\documentclass[letterpaper, 10 pt, conference]{article}  % Comment this line out
\usepackage{graphics} % for pdf, bitmapped graphics files
\usepackage{epsfig} % for postscript graphics files
\usepackage{mathptmx} % assumes new font selection scheme installed
\usepackage{times} % assumes new font selection scheme installed
\usepackage{amsmath} % assumes amsmath package installed
\usepackage{amssymb}  % assumes amsmath package installed

\usepackage{mathtools}

\usepackage{algorithm}
\usepackage{algorithmic}
\usepackage{xcolor}
\usepackage{hyperref}
\usepackage{amsthm}
\usepackage[left=1in, right=1in, bottom=.8in, top=0.8in, headsep=0in, footskip=.2in]{geometry}

\newtheorem{theorem}{Theorem}
\newtheorem{proposition}{Proposition}
\newtheorem{definition}{Definition}
\newtheorem{assumption}{Assumption}

\newtheorem{remark}{Remark}
\newtheorem{fact}{Fact}
\newcommand{\R}{\mathbb{R}}

\newcommand{\M}{\mathcal{M}}
\newcommand{\N}{\mathcal{N}}
\newcommand{\W}{\mathcal{W}}

\title{\LARGE \bf
Manifold-Guided Lyapunov Control with Diffusion Models
}

\author{Amartya Mukherjee, Thanin Quartz, and Jun Liu% <-this % stops a space
\thanks{Amartya Mukherjee, Thanin Quartz, and Jun Liu are with the Department of Applied Mathematics, University of Waterloo, Waterloo, Ontario, Canada N2L 3G1 (email: {\tt\small (a29mukhe, tquartz, j.liu)@uwaterloo.ca}).}%
}
\date{}

\begin{document}

\maketitle
\thispagestyle{empty}
\pagestyle{empty}

%%%%%%%%%%%%%%%%%%%%%%%%%%%%%%%%%%%%%%%%%%%%%%%%%%%%%%%%%%%%%%%%%%%%%%%%%%%%%%%%
\begin{abstract}

This paper presents a novel approach to generating stabilizing controllers for a large class of dynamical systems using diffusion models. The core objective is to develop stabilizing control functions by identifying the closest asymptotically stable vector field relative to a predetermined manifold and adjusting the control function based on this finding. To achieve this, we employ a diffusion model trained on pairs consisting of asymptotically stable vector fields and their corresponding Lyapunov functions. Our numerical results demonstrate that this pre-trained model can achieve stabilization over previously unseen systems efficiently and rapidly, showcasing the potential of our approach in fast zero-shot control and generalizability.
\end{abstract}

%%%%%%%%%%%%%%%%%%%%%%%%%%%%%%%%%%%%%%%%%%%%%%%%%%%%%%%%%%%%%%%%%%%%%%%%%%%%%%%%
\section{INTRODUCTION}

One of the longstanding challenges in nonlinear systems and control is the construction of Lyapunov functions. While Lyapunov functions can essentially be characterized by solutions to partial differential equations (PDEs) and neural network solutions to such PDEs can effectively provide Lyapunov functions \cite{liu2023physics}, solving PDEs for each system can still be time-consuming. Another challenge in computing neural Lyapunov functions is the inclusion of a satisfiability modulo theories (SMT) solver for counterexample-guided training. Verifying that a neural network satisfies all the conditions necessary to be a Lyapunov function is often computationally expensive and significantly increases the training time of such methods. Additionally, the neural Lyapunov function only aids in the verification of a single control system, thus failing to take into account possible uncertainties in the model parameters or the ability to generalize into systems with slightly different dynamics. This could pose difficulties when testing the controller in real-life systems. 

Generative models, such as diffusion models \cite{ho2020denoising,song2020score}, map inputs from a noisy space to the data space. In more recent works, diffusion models have learned to transform noisy inputs or boundary conditions into solutions for PDEs, including the Poisson equation and the Navier-Stokes equation, without incorporating any information about the PDEs into their loss function \cite{apte2023diffusion,ovadia2023ditto}. Furthermore, denoising diffusion restoration models (DDRMs) \cite{kawar2022denoising,murata2023gibbsddrm} can restore clean data in linear inverse problems using a pre-trained diffusion model, eliminating the need for fine-tuning. 

In this paper, we investigate the potential of using pre-trained generative diffusion models to generate Lyapunov functions and stabilizing controllers for nonlinear systems. More specifically, we aim to train a diffusion model, \(G\), that generates pairs of asymptotically stable vector fields and their corresponding Lyapunov functions. This simplifies the control problem to finding a controller such that the closed-loop system possesses a Lyapunov-stable vector field, indicating it falls within the range of \(G\). The Lyapunov function produced by \(G\) further assists in verifying stability for the control problem. This approach offers several advantages over existing learning-based methods for Lyapunov functions (e.g., \cite{chang2019neural,zhou2022neural}): (a) it can significantly reduce the computational effort required to identify a stabilizing controller, (b) it has the capability to generalize to control problems not encountered in the training data, and (c) it can offer stability guarantees with the generated Lyapunov functions.

We propose \textbf{M}anifold \textbf{G}uided \textbf{L}yapunov \textbf{C}ontrol (\textbf{MGLC}), a novel approach to feedback control problems based on diffusion models. We assume that the stabilized controlled system is drawn from a manifold $\M$ that contains asymptotically stable dynamical systems. We also introduce a manifold $\W$ that contains all possible configurations of the dynamical system with different control functions. This simplifies the control problem to finding points where $\M$ and $\W$ intersect. We train a diffusion model with pairs of asymptotically stable vector fields and their Lyapunov functions to estimate $\M$ and incorporate the control design process into the reverse diffusion process. At every time step, we estimate the projection of our vector field into $\M$ by computing Tweedie's estimate, and we update the parameters of the controller to minimize a loss function that reduces the distance between our vector field between the two vector fields. 

Our proposed solution to this problem significantly reduces the computational cost of finding a Lyapunov function. Rather than training a neural network for a specific system, often facilitated by an SMT solver for generating counterexamples \cite{ahmed2020automated}, we use a pre-trained diffusion model to directly output a Lyapunov function that verifies the convergence of our system. As a result, this approach will be useful for solving a large class of control problems rather than a single control problem, thus facilitating zero-shot control, where we achieve stabilization of unseen systems.
We introduce theoretical guarantees that our reverse diffusion sampling method will converge to the manifold $\M$, thus concluding that we have derived a stabilizing controller.
% Our focus in this project will be on 2D vector fields projected into a square grid, but we intend to explore higher dimensional vector fields on future projects using neural operators \cite{chen2023deep} and on safety learning using Lyapunov barrier functions \cite{gurriet2018towards}.
We train the generator on a dataset of pairs of asymptotically stable vector fields and their Lyapunov functions.
Finally, we test our method on four control problems whose dynamics are not in the dataset and are therefore unseen. 
%We hope to get access to NVIDIA GPUs to improve the training process of the generator and to integrate it into our training framework for the real-world system, enhancing computational efficiency.

%%%%%%%%%%%%%%%%%%%%%%%%%%%%%%%%%%%%%%%%%%%%%%%%%%%%%%%%%%%%%%%%%%%%%%%%%%%%%%%%

\subsubsection*{Notation}

Denote $X$ as the discretization of the domain into a grid and $f(X,u(X))$ as the vector field $f$ projected onto the grid $X$ with controller $u(\cdot)$. We use $\psi=(\psi_1,...,\psi_n)$ to refer to the learnable parameters for a controller $u^{(\psi)}(\cdot)$, where $n$ is the number of learnable parameters. Subsequently, $\psi_t$ are the parameters of the control system at time step $t$ in the denoising diffusion process.
Therefore, $x_t:=f(X,u^{(\psi_t)}(X))$ denotes the projection of the vector field $f(\cdot,u^{(\psi_t)}(\cdot))$ on the gridpoints $X$ at time step $t$ in the denoising diffusion process, where the parameters of the controller are $\psi_t$. Lastly, $x_{0|t}:=E[x_0|x_t]$ refers to the estimate of the asymptotically stable vector field $x_0$ conditioned on the vector field at time step $t$, $x_t$. This estimate is computed using Tweedie's formula, which we will formally define in Section \ref{sec:diffusion}. 

%%%%%%%%%%%%%%%%%%%%%%%%%%%%%%%%%%%%%%%%%%%%%%%%%%%%%%%%%%%%%%%%%%%%%%%%%%%%%%%%
\section{PRELIMINARIES AND PROBLEM FORMULATION}

\subsection{Lyapunov Stability}

Throughout this work we consider a control system of the following form 
\begin{equation}\label{eq:2.1}
    \dot{x}=f(x, u),\quad x(0)=x_0,
\end{equation}
where $x \in D$ is the state space of the system, $D \subset \mathbb{R}^n$ is an open set containing the origin and $u \in U \subset \mathbb{R}^m$ is the feedback control input given by $u=u(x)$, where $u(\cdot)$ is a continuous function of the state. Moreover, we assume that $f$ is a locally Lipschitz continuous vector field, $u(\cdot)$ is locally Lipschitz with $u(0)=0$, and $f(0, 0)=0$.
\begin{definition} (Asymptotic Stability)
    The system \eqref{eq:2.1}, under a feedback controller $u=u(x)$, is asymptotically stable at the origin if for any $\epsilon\in\R^+$, there exists $\delta(\epsilon)\in\R^+$ such that if $||x_0||<\delta$, then $||x(t)||<\epsilon$ for all $t\geq 0$. Furthermore, there exists some $\rho>0$ such that $||x_0||<\rho$  implies $\lim_{t\to\infty}||x(t)||=0$.
\end{definition}

\begin{definition}[Lie Derivatives]
    Let $D\subset\R^n$. The Lie derivative of a continuously differentiable scalar function $V:D\to\R$ over a vector field $f(\cdot,u(\cdot))$ and feedback controller $u(\cdot)$ is defined as
    \begin{equation}
        \dot V(x)=\sum_{i=1}^{n}\frac{\partial V}{\partial x_i}\dot x_i=\sum_{i=1}^{n}\frac{\partial V}{\partial x_i}f_i(x,u(x))
    \end{equation}
\end{definition}

\begin{theorem}
    If there exists a continuously differentiable scalar function $V(x)$ satisfying $V(0)=0$, $V(x)>0$ if $x\neq 0$, and $\nabla V(x)\cdot f(x,u(x))<0$, for a given feedback controller $u$, then the system is asymptotically stable (or Lyapunov-stable) at the origin, and $V$ is called a Lyapunov function \cite{chang2019neural}.
\end{theorem}

Consequently, in control problems $\dot x=f(x,u)$, a feedback control function $u=k(x)$ is a stabilizing controller if $\dot x=f(x,k(x))$ is an asymptotically stable system. The existence of a Lyapunov function verifies the stability of the system. However, verifying the Lyapunov conditions is a difficult task that can be handled by SMT solvers as in  \cite{chang2019neural,zhou2022neural} to identify regions where the neural network $V$ violates the conditions for it to be a Lyapunov function and a penalty term for those regions is added to its loss function.

\subsection{Diffusion Models}
\label{sec:diffusion}

Denoising diffusion models \cite{song2019generative,song2020score} are among the current leading methods for generative modeling \cite{ho2020denoising, song2020denoising}. They have shown great success in applications such as the generation of images, speech, and video, as well as image super-resolution \cite{song2020denoising, yang_diffusion_2023}. Other applications include physics-guided human motion (e.g., PhysDiff \cite{yuan_physdiff_2023}), customized ODE solvers that are more efficient than Runge-Kutta methods \cite{lu_dpm-solver_2022}, molecule generation \cite{pmlr-v162-hoogeboom22a}, and more \cite{yang_diffusion_2023}. Furthermore, they are stable to train and are relatively easy to scale \cite{yang_diffusion_2023}. 

Diffusion models consist of a forward process where data is iteratively corrupted by adding noise. This is modeled by the following process:
\begin{equation}\label{eq:forward}
    x_t=\sqrt{\overline{\alpha}_t}x+\sqrt{1-\overline{\alpha}_t}\epsilon_t, \quad t\in[0,T],
\end{equation}
where $\overline{\alpha}_t>0$ is a scaling parameter that monotonically decreases with $t$, $\epsilon_t\sim\N(0,I)$ is Gaussian noise, and $x\sim p_{\text{data}}$ is clean data. The goal of a diffusion model is to reverse this process. This is done by introducing a denoising model $\epsilon_\theta(x_t,t)$ trained with the loss function $\|\epsilon_t-\epsilon_\theta(x_t,t)\|^2$ to estimate the Gaussian noise. This model approximates the score function $\nabla_{x_t}\log p_t(x_t)$.

This methodology allows us to generate realistic data. Particularly, denoising diffusion implicit models (DDIM) \cite{song2020denoising} perform sampling using a two-step procedure for multiple time steps. The first step involves computing $x_{0|t}:=E[x_0|x_t]$ using Tweedie's estimate, which we will define below.

\begin{definition}[Tweedie's estimate \cite{efron2011tweedie,chung2022improving}]
    The Tweedie's estimate for $E[x_0|x_t]$ governed by the process in equation \eqref{eq:forward} is given by the following expression:
    \begin{equation}\label{eq:Tweedie}
        x_{0|t}=\frac{x_t-\sqrt{1-\overline{\alpha}_t}\epsilon_\theta(x_t,t)}{\sqrt{\overline{\alpha}_t}}.
    \end{equation}
\end{definition}

The DDIM algorithm consists of the estimate of $x_{0|t}$ in equation \eqref{eq:Tweedie} coupled with the estimate of $x_{t-1}$ by the following process:
\begin{equation}\label{eq:DDIM}
    x_{t-1}=\sqrt{\overline{\alpha}_{t-1}}x_{0|t}+\sqrt{1-\overline{\alpha}_{t-1}}\epsilon_\theta(x_t,t).
\end{equation}
This process is repeated for multiple time steps. Note as we denoted previously, $x_t$ denotes the discretization of the right-hand side of (\ref{eq:2.1}) on grid points $X$ at time $t$ in the denoising diffusion process. Therefore, (\ref{eq:DDIM}) says we are updating discretized values of the right hand corresponding to updated parameters of the controller.

\subsection{Diffusion on Manifolds}

Diffusion models for inverse problems have gained interest with the advent of denoising diffusion restoration models (DDRM) \cite{kawar2022denoising}, originally used to deblur images by assuming that the original clean image was drawn from a pre-trained unconditional diffusion model. This field, otherwise known as guided diffusion, has been improved upon with diffusion posterior sampling \cite{chung2023diffusion}, partially collapsed Gibbs sampler \cite{murata2023gibbsddrm}, and manifold constraints \cite{chung2023diffusion}.

The problem of guided diffusion involves sampling from a conditional score function $\nabla_{x_t}\log p(x_t|y)$ conditioned on a variable $y$ instead of sampling from the data distribution $\nabla_{x_t}\log p(x_t)$ as done in the pre-trained diffusion model. In many applications, this conditioning $y$ can be a text input \cite{rombach2022high}, an image input, or an image that we intend to restore \cite{murata2023gibbsddrm}. This conditional score function can be separated into two terms using Bayes' theorem: $\nabla_{x_t}\log p(x_t|y)=\nabla_{x_t}\log p(x_t)+\nabla_{x_t}L_t(x_t;y)$. Sampling from $\nabla_{x_t}\log p(x_t)$ is done by our diffusion model, but the sampling procedure $\nabla_{x_t}L_t(x_t;y)$ must be carefully chosen.

The authors of \cite{chung2023diffusion} assume that the data distribution has support in a low-dimensional manifold, known as the manifold hypothesis. They exploit this property to enable conditional sampling on an unconditionally trained diffusion model, thus achieving state-of-the-art performance in inpainting, colorization, and reconstruction tasks in images.

The work of \cite{he2023manifold} builds upon this literature by using tangent spaces on the data manifold $\M$ to return improved estimates of the original clean data. This method shows state-of-the-art results in image quality and image restoration speed.

\subsection{Problem Formulation}

In this paper, we will consider nonlinear control systems of the following form
\begin{equation}
    \dot x=f(x,u), \quad x(0)=x_0,
\end{equation}
where $x\in D$ is the state of the system and $u(\cdot)$ is the feedback control input. Our goal is to derive a controller $u(\cdot)$ such that the trajectory of the system converges to the origin for all $x_0\in D$.

% \begin{assumption}
%     The function $f(x,u)$ is Lipschitz continuous.
%     \begin{equation}
%         \|f(x,u)-f(y,v)\|\leq L\|(x,u)-(y,v)\|
%     \end{equation}
% \end{assumption}

%%%%%%%%%%%%%%%%%%%%%%%%%%%%%%%%%%%%%%%%%%%%%%%%%%%%%%%%%%%%%%%%%%%%%%%%%%%%%%%%
\section{MANIFOLD-GUIDED LYAPUNOV CONTROL (MGLC)}
Define the manifold $\M$ of asymptotically stable control vector fields as 
\begin{equation} \label{M}
    \M = \{ f(X, u(X)) : u \text{ is a stabilizing state feedback controller} \}.
\end{equation}
We build upon past works in diffusion on manifolds by constructing a manifold $\W = \{f(X,u^{(\psi)}(X)):\psi\in\Psi\}$ which is the set of vector fields with our parameterized control input function $u^{(\psi)}(\cdot)$. Therefore, we want to learn a controller $u^{(\psi)}$ such that $f(X,u^{(\psi)}(X)) \in \M \cap \W$. A schematic view is posted in Figure \ref{fig:schematic_view}, comparing DDIM \cite{song2020denoising}, MPGD \cite{he2023manifold}, and our proposed method (MGLC). 

\begin{figure}
    \centering
    \includegraphics[width=\linewidth,trim={0 4cm 3cm 1cm},clip]{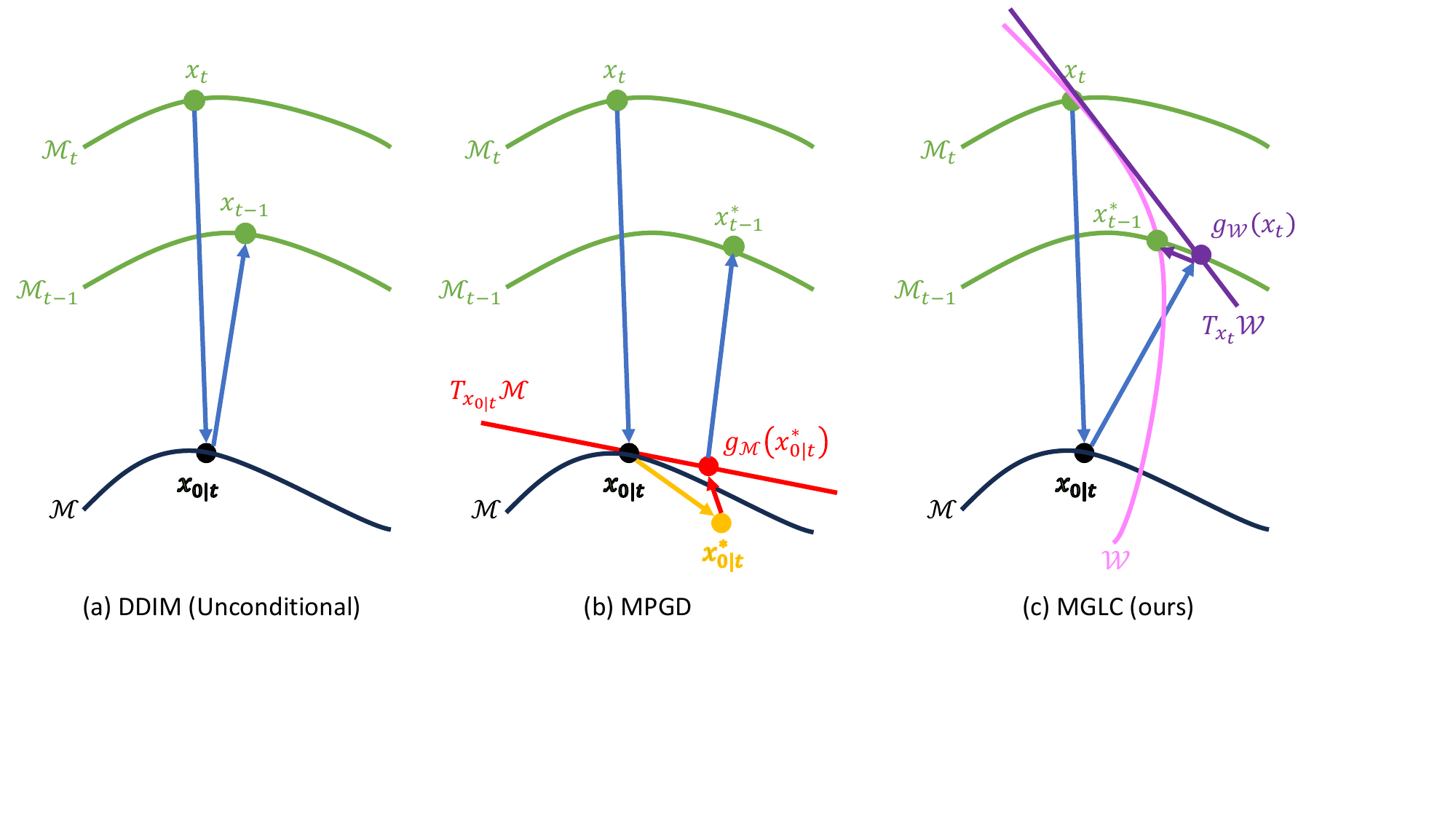}
    \caption{A schematic view of MGLC. Our proposed method introduces a new manifold $\W$ corresponding to the set of vector fields that can be returned by modifying our control input function.}
    \label{fig:schematic_view}
\end{figure}

\subsection{Proposed Method}

We introduce a manifold hypothesis for our problem that assumes our dataset of interest, that is discretized values of the set of asymptotically stable vector fields, lies on a manifold that is low-dimensional compared to the set of all functions defined in a domain containing $X$ projected onto $X$.

\begin{assumption}[Manifold Hypothesis]
    The support $\mathcal{X}$ of the set of asymptotically stable vector fields projected onto a grid $X$ lies on a $k$-dimensional manifold $\M\subset\R^d$ with $k<<d$.
\end{assumption}

Compared to \cite{he2023manifold}, we introduce a new manifold $\W$, which constrains the vector fields our model can output to the set of vector fields with all possible configurations of parameters for our control function. 
% We assume this manifold is diffeomorphic to a linear manifold $\Psi$ containing all the parameter configurations of the control function.

\begin{assumption}\label{as:decoder}
    The manifold $\W=\{f(X,u^{(\psi)}(X)):\psi\in\Psi\}$ contains the vector field $f$ with all possible control parameter configurations.
    We assume $\W\cap\M_t\neq\emptyset$ for all $t=0,...,T$, in other words, $\W$ intersects every manifold $\M_t$.
    % We also assume that for all vector fields $f(\cdot,u^{(\psi)}(\cdot))\in\W$, there exists a perfect decoder $g:\W\to\Psi$ such that $g(f(\cdot,u^{(\psi)}(\cdot)))=\psi$.
    % $\Psi$ is a linear manifold.
\end{assumption}

We also assume that the manifolds $\M_t$ and $\M_{t-1}$ are "close enough" so that a path along the tangent of $\W$ from $\M_t\cap\W$ intersects $\M_{t-1}$ close to a point in $\M_{t-1}\cap\W$.

\begin{assumption}\label{as:close_enough}
    For any time step $t=1,...,T$, the manifolds $\M_t$ and $\M_{t-1}$ are "close enough" with respect to the manifold $\W$ such that a first-order path along $\W$ is a good approximation.
    Rigorously, for every point $x_t\in\M_t\cap\W$, there exists a smooth path $\gamma:[0,1]\to\W$ such that $\gamma(0)=x_t$ and $\gamma(1)=y$, where $y$ is a point in $\M_{t-1}$ satisfying two properties:
    \begin{enumerate}
        \item The path $\gamma$ is entirely contained in the tangent space $T_{x_t}\W$. This means $y$ is contained in $T_{x_t}\W\cap\M_{t-1}$.
        \item The Euclidean distance $d(y,\M_{t-1}\cap\W)$ is less than $\epsilon$ for an arbitrarily small $\epsilon>0$.
    \end{enumerate}
\end{assumption}

We then define a notion of probabilistic concentration of noisy data on a manifold to aid in our convergence results in this paper. This was introduced in the work of \cite{chung2022improving} and then improvised in the paper by \cite{he2023manifold} to fit with the discrete-time setting of diffusion models.

\begin{proposition}[Probabilistic Concentration \cite{chung2022improving,he2023manifold}]\label{prop:prob_con}
    The distribution of noisy data $p_i(x_i)=\int p(x_i|x)p_0(x)dx$ with $p(x_i|x)\sim\N(\sqrt{\overline{\alpha}_t}x,(1-\overline{\alpha}_t)I)$ is concentrated on the $(d-1)$-dimensional manifold $\M_t:=\{y\in\R^n:d(y,\sqrt{\overline{\alpha}_t}M)=\sqrt{(1-\overline{\alpha}_t)(d-k)}$. Rigorously, for all $\epsilon>0$, there exists a $\delta>0$ such that $p_t(B_{\epsilon\sqrt{(1-\overline{\alpha}_t)(d-k)}}(\M_t))>1-\delta$.
\end{proposition}

We compute the estimate of $x_0$ conditioned on $x_t$, $x_{0|t}:=E[x_0|x_t]$ by computing the Tweedie's estimate
\begin{equation}
    x_{0|t}=\frac{1}{\sqrt{\overline{\alpha}_t}}(x_t-\sqrt{1-\overline{\alpha}_t}\epsilon_\theta(x_t,t))\label{eq:update1}
\end{equation}
We then update the parameters $\psi$ using the following gradient update
\begin{equation} 
    \psi_{t-1}=\psi_t-c_t\nabla_\psi L_t(x_t;\psi)\label{eq:update2},
\end{equation}
with some loss function $L_t(x_t;\psi)$. Note that Tweedie's estimate is used in the computation of the loss function as described in Algorithm \ref{alg:mglc}. Finally after updating the parameters $\psi$, we can estimate $x_{t-1}$:
\begin{align}
    x_{t-1}&=f(X,u^{\psi_{t-1}}(X))\label{eq:update3}.
\end{align}

\subsection{Convergence Proofs}

We introduce the following theorems that highlight our main theoretical contribution in this paper. They compute the loss function $L_t(x_t;\psi_t)$ and step size $c_t$ and provides convergence guarantees in the sense that the vector field $f(X, u^{\psi_{0}}(X))$ returned by our proposed methodology has marginal distribution concentrated about $\W \cap \M$ and therefore $u^{\psi_{0}}(X)$ is a stabilizing controller for the system. We first introduce a lemma that explains the choice of $L_t(x_t;\psi)$ and $c_t$. 

\begin{proposition} \label{prop:formulas}
    Fix $t$ and assume $f(X, u^{\psi_{t}}(X)) \in \M_t$. Then with loss function 
    \begin{align}\label{eq:L_t}
        L_t(x_t;\psi)&=\left\|f(X,u^{(\psi_{t})}(X))-\frac{\sqrt{\overline{\alpha}_{t-1}}-\frac{\sqrt{1-\overline{\alpha}_{t-1}}\sqrt{\overline{\alpha}_{t}}}{\sqrt{1-\overline{\alpha}_{t}}}}{1-\frac{\sqrt{1-\overline{\alpha}_{t-1}}}{\sqrt{1-\overline{\alpha}_{t}}}}x_{0|t}\right\|^2 \nonumber \\
        &=: \left \| r_t \right \|^2,
    \end{align}
    where we assume that $r_t$ lies on the tangent space of $\W$ at the point $f(X, u^{\psi_{t}}(X))$,
    and step size 
    \begin{equation}\label{eq:c_t}
        c_t= (1-\frac{\sqrt{1-\overline{\alpha}_{t-1}}}{\sqrt{1-\overline{\alpha}_{t}}})
    \end{equation}
    we have that there exists $x_{t-1}$ probabilistically concentrated in $\M_{t-1}$ such that
    \begin{equation}
        x_{t-1} = f(X, u^{\psi_{t}}(X)) - c_t r_t.
    \end{equation}
\end{proposition}
\begin{proof}
    By the assumption that $x_t:=f(X,u^{(\psi_{t})}(X))\in\M_{t}$, this vector field lies on the manifold of the diffusion model at time $t$ and therefore by Proposition \ref{prop:prob_con}, there exists a scaling constant $\overline{\alpha}_t$ and an asymptotically stable vector field $x \in \M$ such that $x_t=\sqrt{\overline{\alpha}_t}x+\sqrt{1-\overline{\alpha}_t}\epsilon$. As the diffusion model is given by the DDIM algorithm as stated in \eqref{eq:Tweedie} and \eqref{eq:DDIM}, this states that there exists an asymptotically stable vector field $x_{t-1} \in \M_{t-1}$, concentrated on $\M_{t-1}$ with high probability, such that 
    \begin{equation}
        x_{t-1}=\sqrt{\overline{\alpha}_{t-1}}x_{0|t}+\sqrt{1-\overline{\alpha}_{t-1}}\frac{x_t-\sqrt{\overline{\alpha}_{t}}x_{0|t}}{\sqrt{1-\overline{\alpha}_{t}}}.
    \end{equation}
    Substituting $x_{t}=f(X,u^{(\psi_{t})}(X))$ and $x_{t-1}=f(X,u(X))$ gives
    \begin{align} \label{connection}
        f(X,u(X))=&\frac{\sqrt{1-\overline{\alpha}_{t-1}}}{\sqrt{1-\overline{\alpha}_{t}}}f(X,u^{(\psi_{t})}(X))\nonumber\\
        &+\left(\sqrt{\overline{\alpha}_{t-1}}-\frac{\sqrt{1-\overline{\alpha}_{t-1}}\sqrt{\overline{\alpha}_{t}}}{\sqrt{1-\overline{\alpha}_{t}}}\right)x_{0|t}
    \end{align}
    By a first-order approximation, 
    \begin{equation} \label{eqn:h}
        f(X,u(X))=f(X,u^{(\psi_{t})}(X))-\nabla_\psi f(X,u^{(\psi_{t})}(X))h
    \end{equation}
    for some vector $h \in \R^n$. Set
    \begin{align} \label{h}
        -h
        =&[\nabla_\psi f(X,u^{(\psi_{t})}(X))\nabla_\psi f(X,u^{(\psi_{t})}(X))^T]^{-1}\\
        &\nabla_\psi f(X,u^{(\psi_{t})}(X))^T\\
        &\Big[
        (1-\frac{\sqrt{1-\overline{\alpha}_{t-1}}}{\sqrt{1-\overline{\alpha}_{t}}})f(X,u^{(\psi_{t})}(X))\nonumber\\
        &+\left(\sqrt{\overline{\alpha}_{t-1}}-\frac{\sqrt{1-\overline{\alpha}_{t-1}}\sqrt{\overline{\alpha}_{t}}}{\sqrt{1-\overline{\alpha}_{t}}}\right)x_{0|t}
        \Big]. \label{eq:square_bracket_term}
    \end{align}
    Since $f(X,u^{(\psi_{t})}(X)) [\nabla_\psi f(X,u^{(\psi_{t})}(X))\nabla_\psi f(X,u^{(\psi_{t})}(X))^T]^{-1}$ is the matrix representation of the projection onto the column space of $f(X,u^{(\psi_{t})}(X))$ this is the projection onto the tangent space of $\W$ at the point $f(X, u^{\psi_{t}}(X))$. Therefore, as $r_t$ is assumed to lie on the tangent space, the projection matrix
    \begin{align}
        P=&\nabla_\psi f(X,u^{(\psi_{t})}(X))\nonumber\\
        &[\nabla_\psi f(X,u^{(\psi_{t})}(X))\nabla_\psi f(X,u^{(\psi_{t})}(X))^T]^{-1}\nonumber\\
        &\nabla_\psi f(X,u^{(\psi_{t})}(X))^T
    \end{align} acts as the identity matrix on $r_t$. It follows that
    \begin{equation}
        \nabla_\psi f(X,u^{(\psi_{t})}(X))h = (1-\frac{\sqrt{1-\overline{\alpha}_{t-1}}}{\sqrt{1-\overline{\alpha}_{t}}})r_t
    \end{equation}  
    and (\ref{connection}) is satisfied for this choice. Therefore, substituting and simplifying gives
    \begin{equation}
        f(X, u(X)) = f(X, u^{(\psi_t)}(X)) - c_t r_t 
    \end{equation}
\end{proof}

% \begin{theorem}\label{thm:algorithm}
%     Assume $f(X,u^{(\psi_{t})})\in\M_{t}$. Then there exists a loss function $L_t(x_t;\psi)$ and a step size $c_t$ such that
%     $f(X,u^{(\psi_{t-1})}(X))$ is probabilistically concentrated in $\M_{t-1}$, where $\psi_{t-1}=\psi_{t}-\nabla_\psi L(x_t;\psi_t)$
% \end{theorem}

% \begin{proof}
%     Assume $x_t:=f(X,u^{(\psi_{t})}(X))\in\M_{t}$. This means there exists an $x\in\M$ such that $x_t=\sqrt{\overline{\alpha}_t}x+\sqrt{1-\overline{\alpha}_t}\epsilon$, where $\epsilon\sim\N(0,I)$.

%     % The DDIM algorithm as given by equations \eqref{eq:Tweedie} and \eqref{eq:DDIM} returns an $x_{t-1}$, which is probabilistically concentrated in $\M_{t-1}$ as proved by \cite{he2023manifold}. Writing this out gives
    
%     By Lemma \ref{formulas} $x_{t-1}$ 

%     Since the manifold $\Psi$ is linear by Assumption \ref{as:decoder}, this means the gradient $\nabla_\psi L(x_{t};\psi)=\nabla_\psi L(f(X,u^{g(x_{t})}(X));\psi)=\nabla_\psi L(f(X,u^{\psi_t}(X));\psi)$ also lies in the tangent space $T_{x_t}\W$ for any $c_t>0$ as proved in theorem 2 in \cite{he2023manifold}, showing that this first-order discretization offers the exact solution to our estimate of $\psi$. Thus, $f(X,u^{(\psi_{t-1})}$ is probabilistically concentrated on the manifold $\M_{t-1}$.
% \end{proof}
\begin{remark} \label{rem}
    Note that since $r_t$ lies on the tangent plane we see that $x_{t-1}$ lies in $T_{x_t}\W \cap \M_{t-1}$. Therefore, by Assumption \ref{as:close_enough} there exists parameters $\psi_{t-1}$ such that $f(x,u^{\psi_{t-1}}(X))$ lies in an $\epsilon$ neighborhood of $f(X,u(X))$ obtained from Proposition \ref{prop:formulas}. Intuitively, this shows that updating parameters in \eqref{eq:update2} computes the difference in parameters corresponding to a linear shift from $\M_{t} \cap \W$ to $\M_{t-1}$. 
\end{remark}
\begin{theorem}\label{thm:convergence}
    Let $\M$ be the set of asymptotically stable vector fields projected onto a grid. Let $c_t$ be defined as in equation \eqref{eq:c_t}. Assume the diffusion model $\epsilon_\theta$ is optimal.
    With the update rule \eqref{eq:update1}--\eqref{eq:update3} with $L_t(x_t;\psi_t)$ defined as in equation \eqref{eq:L_t}, we can obtain an $x_{t-1}\sim N(x_{t-1};f(X,u^{\psi_{t-1}}(X)),\sigma_t^2I)$ whose marginal distribution
    \begin{equation}
        p(x_{t-1})=\int\N(x_{t-1};f(X,u^{\psi_{t-1}}(X)),\sigma_t^2I)p(x_t|x)p(x)dxdx_t
    \end{equation}
    is probabilistically concentrated on $\M_{t-1}\cap\W$
\end{theorem}

\begin{proof}
    We first acknowledge that there is a one-to-one correspondence between $x_t$ and $\psi_t$ for all $t=0,...,T$, thus showing that $p(x_t)=p(\psi_t)$. We also acknowledge that $x_t\in\W$ for all $t$ since the update rule $x_t=f(X,u^{(\psi_t)}(X))$ forces this property on $x_t$.
    
    We then prove that for all $t$, there exists an $x\in\M\cap\W$ such that the $x_t$ generated from equations \eqref{eq:update1}--\eqref{eq:update3} can also be generated by the forward process of the diffusion model constrained to the $\W$ manifold. In other words, $\psi_t=\sqrt{\overline{\alpha}_t}\psi+\sqrt{1-\overline{\alpha}_t}\epsilon$ for some $\psi\in\W,\epsilon\in\N(0,I)$, where $x=f(X,u^{(\psi)}(X))$. We prove this using induction.

    For the base case, let $t=T$. Since the parameters $\psi_T$ are drawn from a Gaussian prior, $\psi_T=\sqrt{\overline{\alpha}_T}\psi+\sqrt{1-\overline{\alpha}_T}\epsilon=\epsilon$ for any $\psi\in\M$ since $\overline{\alpha}_T=0$.

    The proof for $t\leq T$, has been done in Proposition \ref{prop:formulas} and Remark \ref{rem}. We are guaranteed the existence of parameters $\phi_t$ for any $t \leq T$ such that $f(X, u^{\phi_t}(X))$ lies in an $\epsilon$ neighborhood of some $x_t \in \M_t$.This completes the proof by induction. 

    % For all $t\geq T_1$, suppose there exists an $x\in\M\cap\W$ such that $x_t=\sqrt{\overline{\alpha}_t}x+\sqrt{1-\overline{\alpha}_t}\epsilon$ for some $\epsilon\sim\N(0,I)$.
    % Then since the diffusion model is optimal, we can conclude that $x_{0|T_1}\in\M$.

    % We know that $\sqrt{1-\bar{\alpha}_{T_1-1}-\sigma_{T_1}^2}\epsilon_\theta(x_{T_1},T_1)+\sigma_{T_1}\epsilon_{T_1}=\sqrt{1-\bar{\alpha}_{T_1-1}}\Tilde{\epsilon}$ for some $\Tilde{\epsilon}\sim\N(0,I)$. Note that due to our acknowledgement that $p(x_t)=p(\psi_t)$, we can also conclude that $\epsilon_\theta(x_{T_1},T_1)=\epsilon_\theta(\psi_{T_1},T_1)$. This gives us $\sqrt{1-\bar{\alpha}_{T_1-1}-\sigma_{T_1}^2}\epsilon_\theta(\psi_{T_1},T_1)+\sigma_{T_1}\epsilon_{T_1}=\sqrt{1-\bar{\alpha}_{T_1-1}}\Tilde{\epsilon}$.

    % A continuous-time formulation of $\epsilon_\theta(\psi_{T_1},T_1)$ is $\sigma(t)\nabla_\psi\log(p(\psi))=\sigma(t)\nabla_\psi\log(p(\psi|x_{0|T_1}))=-c_t\nabla_\psi\|\psi-x_{0|t}\|^2$, where $c_t$ is an arbitrarily defined coefficient that varies with $t$. This shows that $x_{T_1-1}$ can be expressed as $\sqrt{\overline{\alpha}_{T_1-1}}x+\sqrt{1-\overline{\alpha}_{T_1-1}}\epsilon$ for some $x\in\M,\epsilon\sim\N(0,I)$. Thus, $x_{T_1-1}\in\M$.
    
    % Since the manifold $\W$ is diffeomorphic to a linear manifold $\Psi$ by assumption \ref{as:decoder}, this means the gradient $\nabla_\psi L(x_{t};\psi)$ also lies in the tangent space $T_{x_t}\W$ for any $c_t>0$. So, $x_{t-1}=f(X,u^{\psi_{t-1}}(X))\in\W$, where $\psi_{t-1}=\psi_{t}-c_t\nabla_\psi L(x_t;\psi)$. Thus, $x_{t-1}\in\M_{t-1}\cap\W$. 
\end{proof}

Proposition \ref{prop:formulas} and Theorem \ref{thm:convergence} allow us to propose a methodology for updating the parameters $\psi$ to design a stabilizing controller.

\subsection{Algorithm}

We provided our proposed algorithm in Algorithm \ref{alg:mglc} below. This algorithm returns the stable vector field $x_0$ and the stabilizing controller $\psi_0$. This algorithm will be tested in four baselines to derive stabilizing controllers.

\begin{algorithm}[H]
	\caption{Manifold Guided Lyapunov Control (MGLC)}\label{alg:mglc} 
	\begin{algorithmic}[1]
		\REQUIRE  $f(\cdot,u(\cdot))$, Pretrained diffusion model $\epsilon_\theta(\cdot,\cdot)$
            \STATE Initialize $\psi_T\sim\N(0,I)$
		\FOR{t=T,...,1}
                \STATE Set $x_t=f(X,u^{(\psi_t)})$
                \STATE Estimate $x_{0|t}$ using \eqref{eq:Tweedie}
    		\STATE Compute $L_t(x_t;\psi_t)$ and $c_t$ using \eqref{eq:L_t} and \eqref{eq:c_t} respectively
    		\STATE Update $\psi_{t-1}=\psi_t-c_t\nabla_{\psi_t}L_t(x_t;\psi_t)$
            \ENDFOR
            \RETURN $x_0=f(X,u^{(\psi_0)})$ and $\psi_{0}$
	\end{algorithmic}
\end{algorithm}

%%%%%%%%%%%%%%%%%%%%%%%%%%%%%%%%%%%%%%%%%%%%%%%%%%%%%%%%%%%%%%%%%%%%%%%%%%%%%%%%
\section{IMPLEMENTATION}

In this section, we explain the dataset and the unconditional diffusion model we trained for this paper. To ensure the reproducibility of our results, we provided our source code in the following GitHub repository: \href{https://github.com/amartyamukherjee/minimal-diffusion}{https://github.com/amartyamukherjee/minimal-diffusion}

An overview of our model is posted in Figure \ref{fig:diagram}. We train the diffusion model to output images with three channels. The first two channels correspond to $f_1(X)$ and $f_2(X)$, the two scalar elements of an asymptotically stable vector field. The third channel corresponds to $V(X)$, the Lyapunov function that verifies their stability.

\subsection{Dataset Generation}

For our diffusion model, we made a dataset of 2,000 pairs of $(f,V)$ in 2D space. The first 1,000 pairs we made have the following format:
\begin{align}
    f(x)&=Ax+\beta_f^T\tanh{(W_fx)}\\
    V(x)&=\beta_V\cdot\tanh{(W_Vx+b_V)},
\end{align}
where $A$ is a $2\times 2$ Hurwitz matrix. Each entry of $\beta_f\in\R^{2\times 20}$ is sampled from $\N(0,0.2^2)$, and each entry of $W_f\in\R^{20\times 2}$ is sampled from $\N(0,1)$. $V(x)$ is a neural Lyapunov function trained identically to the method in \cite{zhou2022neural} to minimize the loss $(\dot V+\|x\|^2)$.

For the next 1,000 pairs, we focus on second-order systems. We use the fact stated below.
\begin{fact}[Example 4.8 of \cite{khalil2002nonlinear}]
    Consider the system
    \begin{align}
        \dot{x}_1&=x_2\\
        \dot{x}_2&=-h_1(x_1)-h_2(x_2),
    \end{align}
    where $h_1(\cdot)$ and $h_2(\cdot)$ are locally Lipschitz and satisfy $h_i(0)=0,yh_i(y)>0,\forall y\in(-a,a)\backslash\{0\}$. This system is asymptotically stable with the following Lyapunov function candidate
    \begin{equation}
        V(x_1,x_2)=\int_0^{x_1}h_1(y)dy+\frac{1}{2}x_2^2,
    \end{equation}
    satisfying
    \begin{equation}
        \dot V=-x_2h_2(x_2).
    \end{equation}
\end{fact}

Using this fact, we can design the following system:
\begin{align}
    \dot x_1&=x_2\\
    \dot x_2&=-c_1x_1-c_2\tanh{x_1}-c_3x_2-c_4\tanh{x_2}\\
    V(x_1,x_2)&=\frac{c_1}{2}x_1^2+c_2\log(\cosh{x_1})+\frac{1}{2}x_2^2,
\end{align}
where $c_i$ is drawn from the uniform distribution $U(0,5)$. This ensures that $\dot V=-x_2(c_3x_2+c_4\tanh{x_2})\le -c_3x_2^2$.

\subsection{Diffusion model}\label{sec:diffusion_PDEgen}

\begin{figure}
    \centering
    \includegraphics[width=\linewidth]{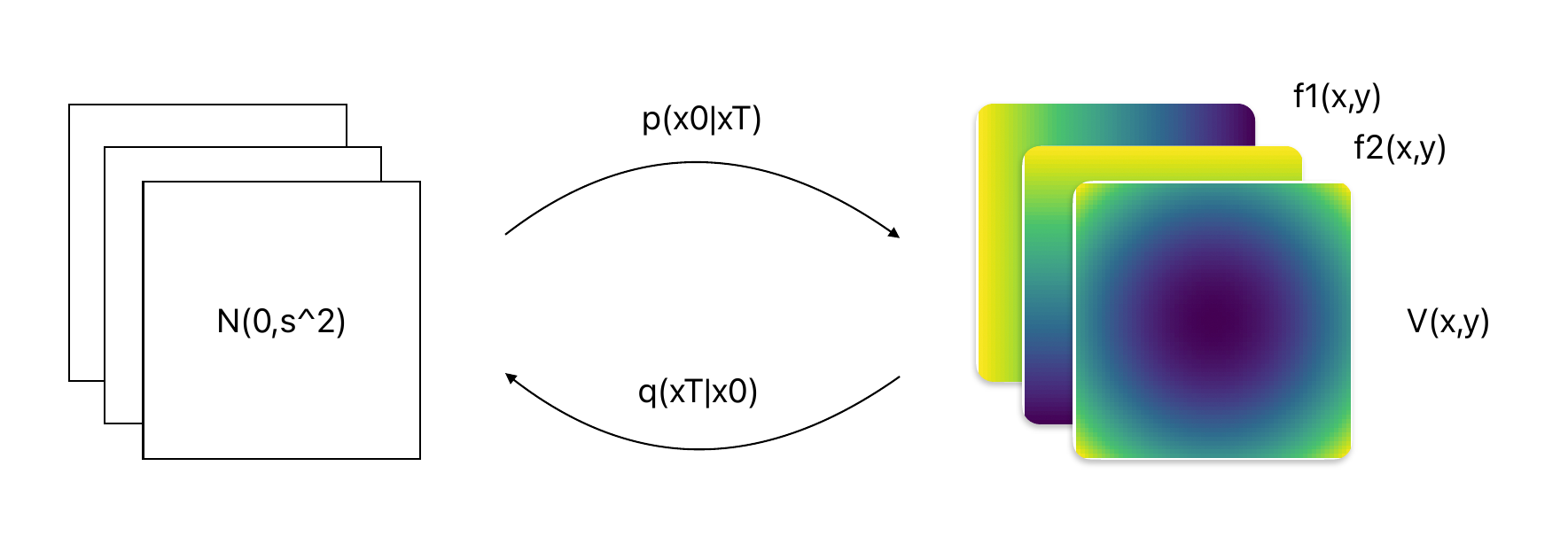}
    \caption{An overview of our model in the 2D control setting. We train the generator to output images with three channels. The first two channels are the scalar elements of an asymptotically stable vector field, $f_1(x,y)$ and $f_2(x,y)$. The third channel is a Lyapunov function $V(x,y)$ that verifies the stability of the vector field.}
    \label{fig:diagram}
\end{figure}

We trained a diffusion model by modifying the GitHub repository by \cite{vsehwagGithub} that is based on DDIM. This model involves a forward (or "diffusion process") that is a Markov chain, which gradually adds Gaussian noise to the data given by a cosine scheduler. Upon training our diffusion model, we needed to normalize our vector fields and Lyapunov function so that each pixel has values in $[-1,1]$. This means that our model is trained with Lyapunov functions that satisfy $\dot V=c\|x\|^2$ for the first 1,000 data points and $\dot V<cx_2^2$ for the second 1,000 data points, where $c$ is a strictly positive constant.

We trained our diffusion model in parallel using two NVIDIA 3090Ti GPUs for our dataset of 2,000 samples and the training time was approximately 8 hours.

\subsection{Controller}

The format of the controller is
\begin{equation}
    u^{(\psi)}(x)=C\sum_{i=1}^n\tanh(\psi_ix_i),
\end{equation}
where $\psi_i,i=1,...,n$ are trainable parameters. At every step of the reverse diffusion process, we update the parameters $\psi$ using the loss function defined in equation \eqref{eq:L_t} and the step size in equation \eqref{eq:c_t}.

%%%%%%%%%%%%%%%%%%%%%%%%%%%%%%%%%%%%%%%%%%%%%%%%%%%%%%%%%%%%%%%%%%%%%%%%%%%%%%%%
\section{NUMERICAL RESULTS}

% \subsection{Learning a stable vector field}

\begin{figure*}[ht]
    \centering
    \includegraphics[width=\linewidth]{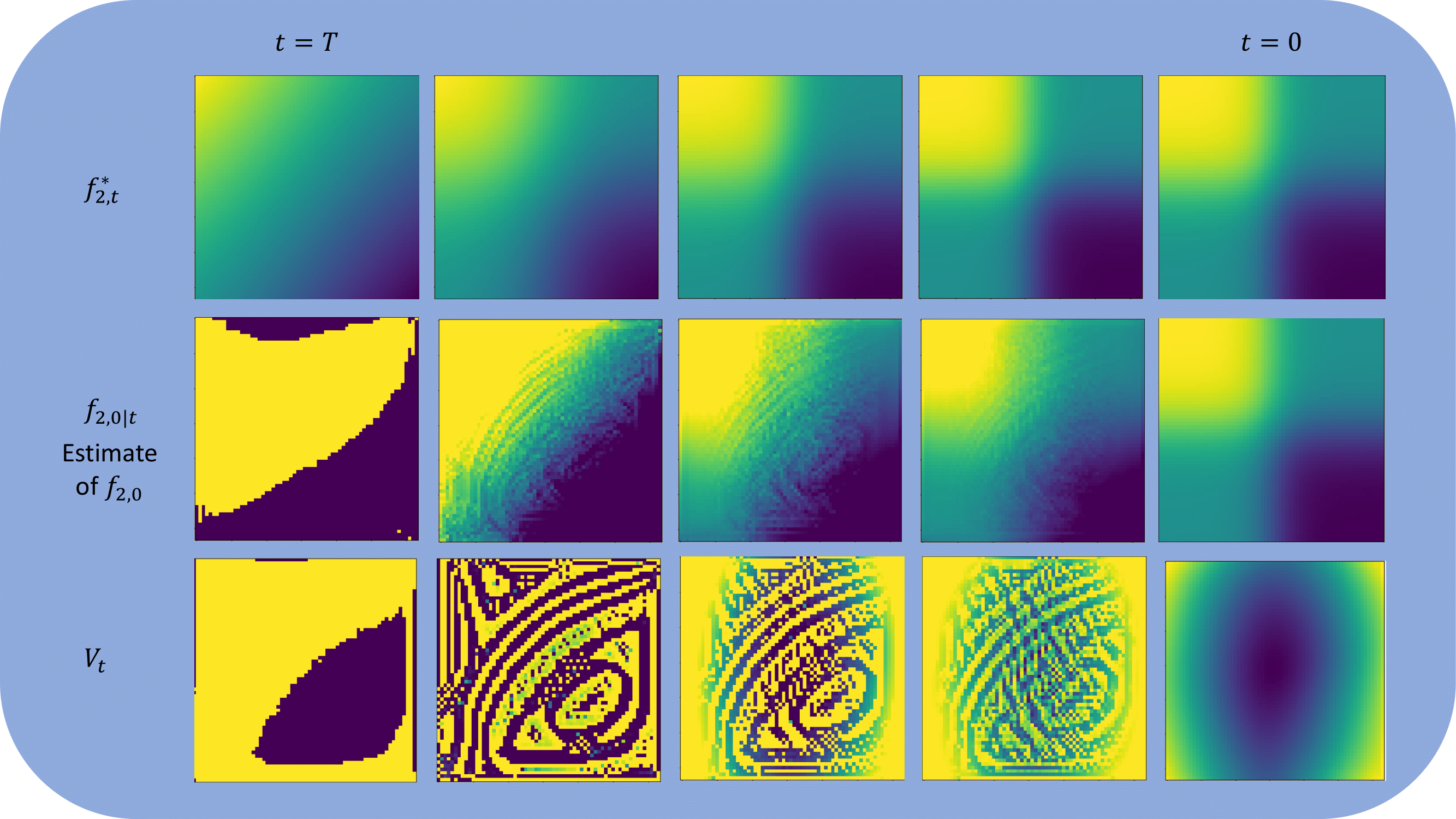}
    \caption{Visualization of MGLC for the inverted pendulum environment}
    \label{fig:MGLC_Visualization}
\end{figure*}

% \subsection{Learning a controller}

We finally used this pre-trained model to derive stabilizing controllers for four unseen nonlinear systems. We notice that the derivation of a controller took only 16 seconds.

\subsection{Inverted Pendulum}

The inverted pendulum is governed by the following dynamics:
\begin{equation}\label{eq:pendulum}
    \ddot\theta=\frac{mgl\sin\theta+u-0.1\dot\theta}{ml^2},
\end{equation}
where our constants are $g=9.81,m=0.15,l=0.5$. Using our MGLC algorithm, our control function is
\begin{equation}
    u(x_1,x_2)=20\tanh(-4.16928\theta) + 20\tanh(-3.14848\dot\theta).
\end{equation}
To verify the stability of this system, we solved it numerically with 100 different initial conditions using RK45 and plotted the trajectories in Figure \ref{fig:pendulum}. The plot shows that all the trajectories converged to the origin, thus showing that our control function is stabilizing.
\begin{figure}[H]
    \centering
    \includegraphics[width=\linewidth]{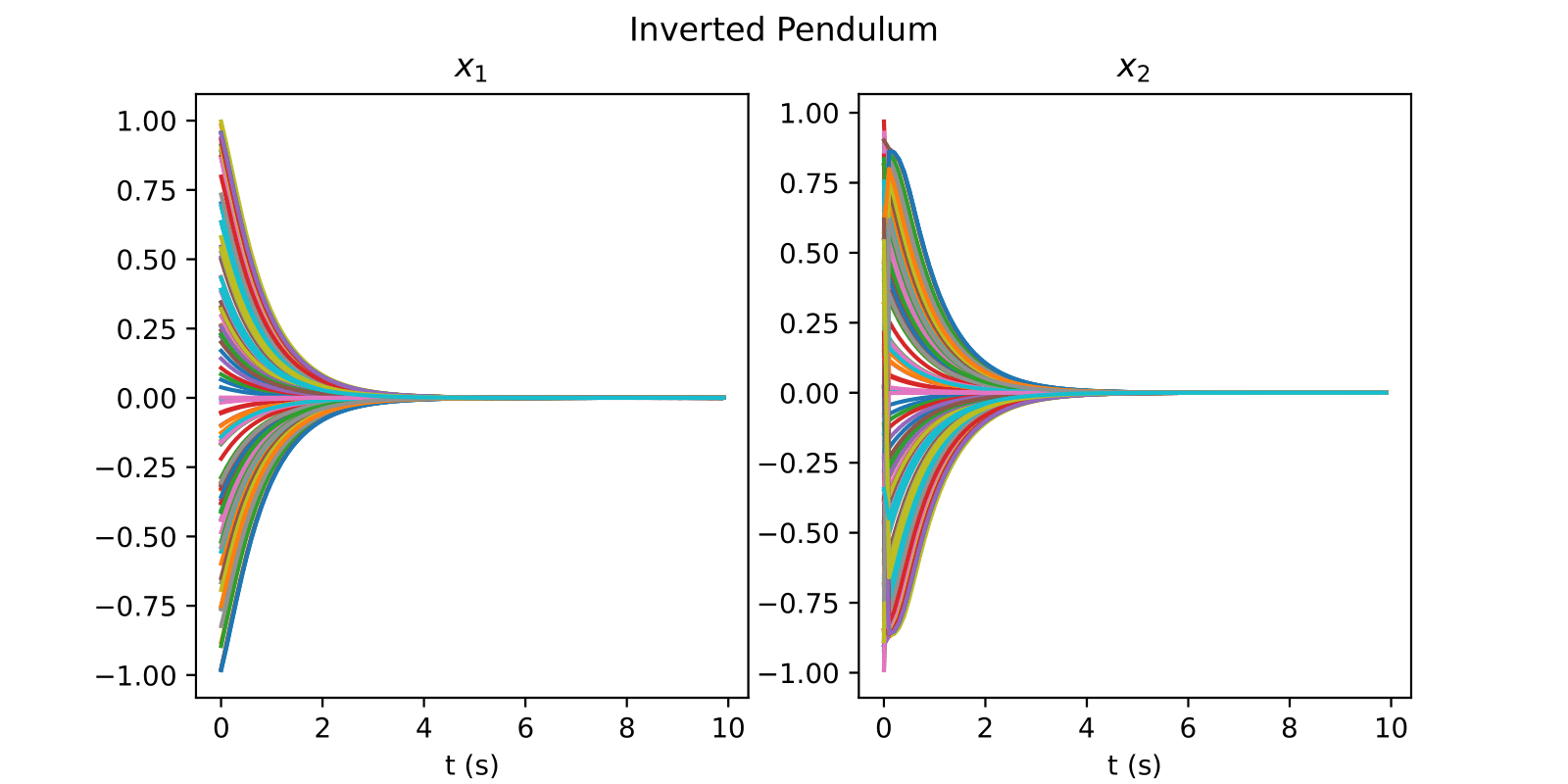}
    \caption{Trajectories of the controlled inverted pendulum system with 100 different randomly sampled initial conditions}
    \label{fig:pendulum}
\end{figure}

We have also posted a visualization of the convergence of our system for 5 time steps in Figure \ref{fig:MGLC_Visualization}. Since $f_{1,t}(x)=x_2$, we decided to omit that channel in this visualization. We have also plotted our estimate of the second channel of $x_{0|t}$, which we will call $f_{2,0|t}$. That channel shows a convergence of our system to a stable but noisy vector field, which is used in our loss function for updating $\psi$. We have also shown the evolution of our candidate Lyapunov function in the third channel.

\subsection{Damped Duffing Oscillator}

The damped Duffing oscillator is governed by the following dynamics:
\begin{align}
    \dot x_1&=x_2\\
    \dot x_2&=-0.5x_2-x_1(4x_1^2-1)+0.5u.
\end{align}
The uncontrolled system $(u=0)$ has two stable equilibrium points, $(-0.5,0)$ and $(0.5,0)$, and an unstable equilibrium point in the origin $(0,0)$. Using our MGLC algorithm, our control function is
\begin{equation}
    u(x_1,x_2)=20\tanh(-3.89859x_1) + 20\tanh(-4.46941x_2).
\end{equation}
The numerical solutions of our system, plotted in Figure \ref{fig:ddo}, show that our control function is stabilizing.
\begin{figure}[H]
    \centering
    \includegraphics[width=\linewidth]{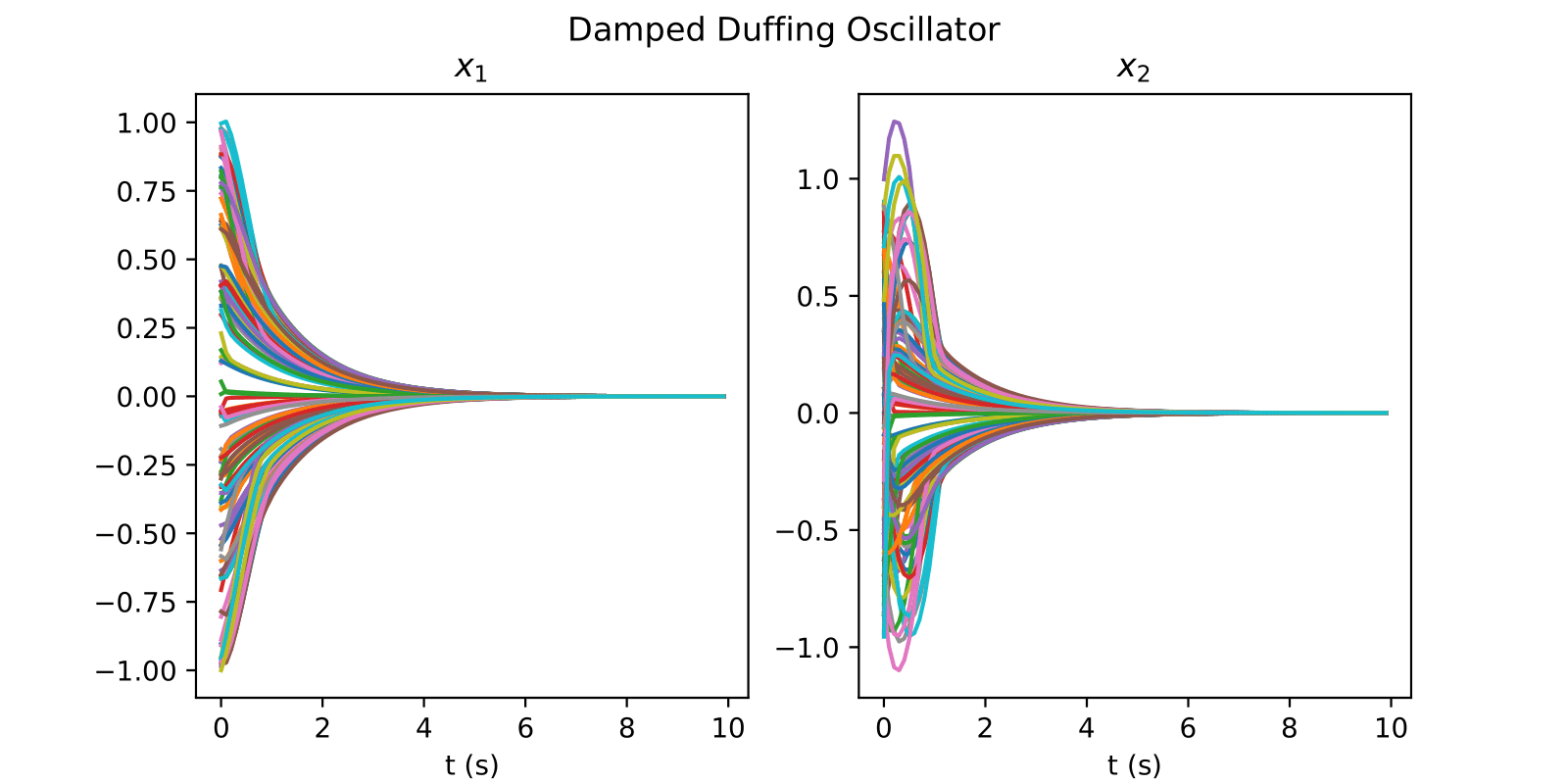}
    \caption{Trajectories of the controlled damped Duffing oscillator system with 100 different randomly sampled initial conditions}
    \label{fig:ddo}
\end{figure}

\subsection{Van Der Pol Oscillator}

The controlled Van Der Pol oscillator system is governed by the following dynamics:
\begin{align}
    \dot x_1&=2x_2\\
    \dot x_2&=-0.8x_1+2x_2-10x_1^2x_2+u.
\end{align}
The uncontrolled system $(u=0)$ shows a limit cycle in a neighborhood near the origin, and an unstable equilibrium point in the origin $(0,0)$. Using our MGLC algorithm, our control function is
\begin{equation}
    u(x_1,x_2)=20\tanh(-5.05384x_1) + 20\tanh(-3.25052x_2).
\end{equation}
The numerical solutions of our system, plotted in Figure \ref{fig:vdp}, show that our control function is stabilizing.
\begin{figure}[H]
    \centering
    \includegraphics[width=\linewidth]{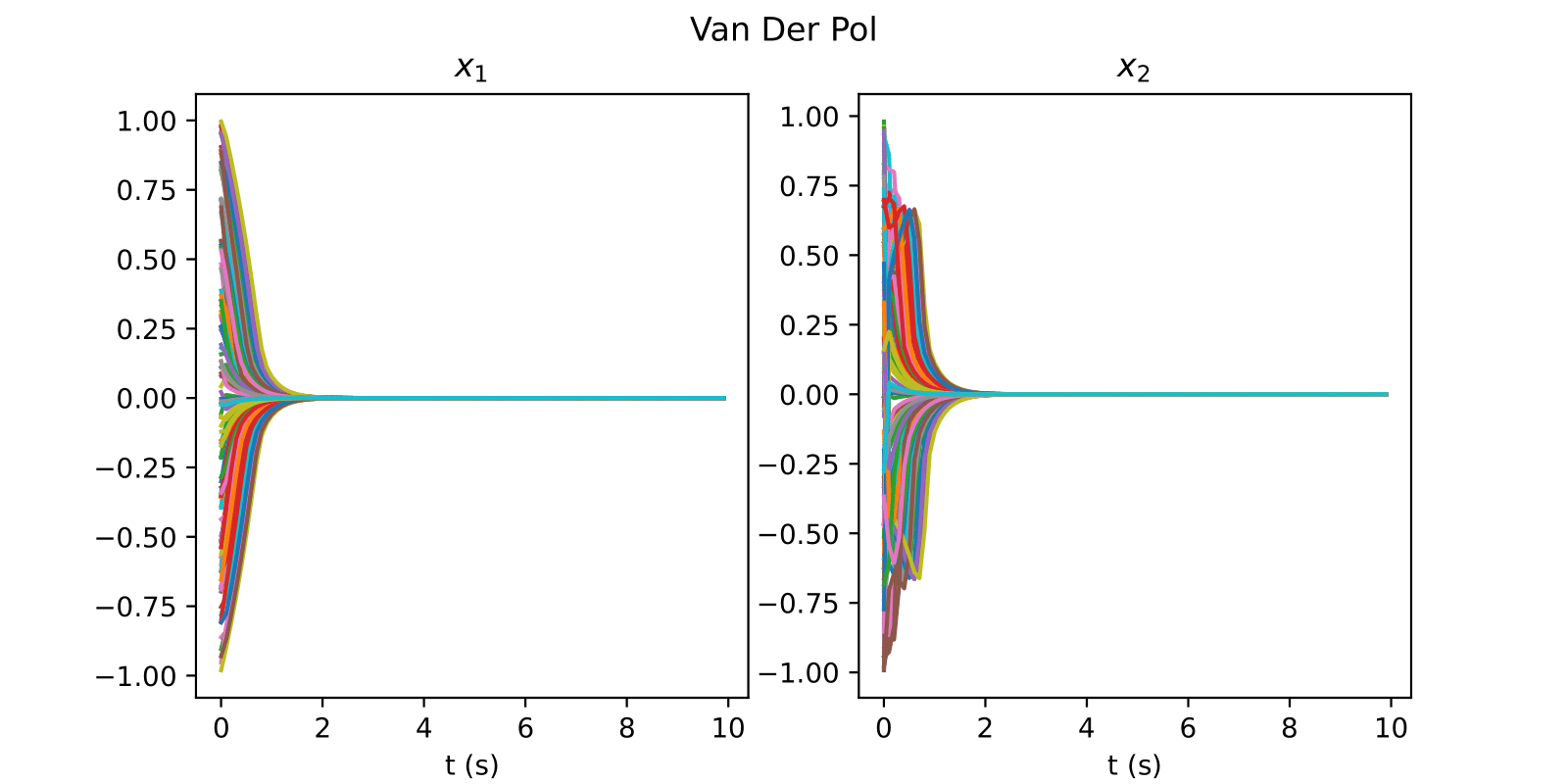}
    \caption{Trajectories of the controlled Van Der Pol oscillator system with 100 different randomly sampled initial conditions}
    \label{fig:vdp}
\end{figure}

\subsection{Noisy Inverted Pendulum}

In this experiment, we consider the pendulum system in equation \eqref{eq:pendulum} with uncertainties in the parameters $m,l$ and the control $u$. At each time step, we replace $m,l$ in the system with $m+z_1,l+z_2$, and the control is given by $(1+z_3)u^{(\psi)}(x)$, where $z_i$ are sampled from the uniform distribution $U(-0.05,0.05)$ for $i=1,2,3$.
\begin{equation}
    u(x_1,x_2)=20\tanh(-4.01703x_1) + 20\tanh(-3.63485x_2).
\end{equation}
To verify the stability of this system, we solved it numerically with 100 different initial conditions using the Euler-Maruyama scheme and plotted the trajectories in Figure \ref{fig:noisy_pendulum}. The plot shows that all the trajectories converged to the origin, thus showing that our control function is stabilizing.
\begin{figure}[H]
    \centering
    \includegraphics[width=\linewidth]{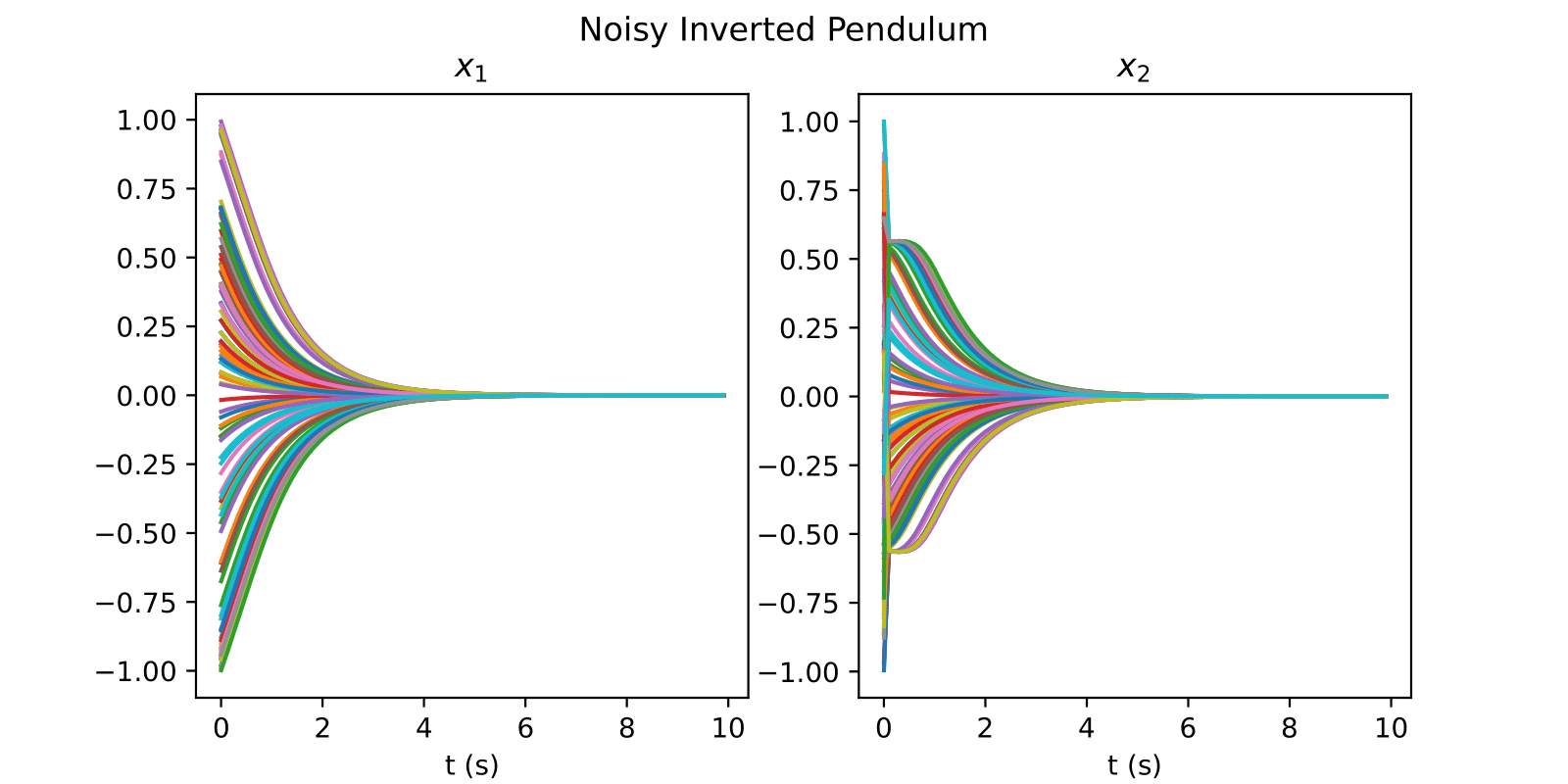}
    \caption{Trajectories of the controlled inverted pendulum system designed with uncertainties in parameters and control inputs on 100 different randomly sampled initial conditions}
    \label{fig:noisy_pendulum}
\end{figure}

%%%%%%%%%%%%%%%%%%%%%%%%%%%%%%%%%%%%%%%%%%%%%%%%%%%%%%%%%%%%%%%%%%%%%%%%%%%%%%%%
\section{CONCLUSIONS AND FUTURE WORKS}

\subsection{Conclusions}

This paper introduces a new approach to the control of nonlinear systems by identifying the closest asymptotically stable vector field relative to a predetermined manifold. We employ diffusion models to aid us in estimating this stable vector field using Tweedie's estimate. We finally use this estimate to update the parameters of our control function. This method helps us achieve fast zero-shot control for control problems that are unseen by the diffusion model. This work, as a result, presents a novel application of diffusion models for transfer learning in control problems. Furthermore, the design of a stabilizing control only needs evaluations of the vector field $f(x,u)$ at grid points $X$, thus showing that this can easily be implemented as a data-driven approach.

\subsection{Future Works}

One key area of further exploration is the extension of this work to systems of high dimensions. Since the diffusion model outputs images in a pixel format, we were able to only work with 2D systems. The extension to high dimensions may be possible using models such as latent diffusion \cite{rombach2022high}. The work of \cite{he2023manifold} has provided improved convergence guarantees using perfect encoders as done in latent diffusion. We expect that, equivalently, improved convergence guarantees can be expected using perfect encoders of high-dimensional systems.

%%%%%%%%%%%%%%%%%%%%%%%%%%%%%%%%%%%%%%%%%%%%%%%%%%%%%%%%%%%%%%%%%%%%%%%%%%%%%%%%
\section{ACKNOWLEDGMENTS}

We sincerely acknowledge the support of Professor Yaoliang Yu from the Department of Computer Science, University of Waterloo for providing us with GPUs to train our model and for his valuable feedback on the paper.

%%%%%%%%%%%%%%%%%%%%%%%%%%%%%%%%%%%%%%%%%%%%%%%%%%%%%%%%%%%%%%%%%%%%%%%%%%%%%%%%

\bibliography{cdc}
\bibliographystyle{plain}

\end{document}